\documentclass{article}
\usepackage{graphicx}
\usepackage{longtable}
\usepackage{float}
\usepackage{wrapfig}
\usepackage{soul}
\usepackage{amssymb}
\usepackage{hyperref}
\usepackage{amsmath}
\usepackage{subfigure}
\usepackage{bm}
\usepackage{algorithm}
\usepackage{algorithmic}
\usepackage{color}

\newtheorem{theorem}{Theorem}[section]

\newenvironment{proof}[1][Proof]{\begin{trivlist}
\item[\hskip \labelsep {\bfseries #1}]}{\end{trivlist}}

\newcommand{\qed}{\nobreak \ifvmode \relax \else
      \ifdim\lastskip<1.5em \hskip-\lastskip
      \hskip1.5em plus0em minus0.5em \fi \nobreak
      \vrule height0.75em width0.5em depth0.25em\fi}

\title{Metric Learning across Heterogeneous Domains by Respectively Aligning Both Priors and Posteriors}
\author{Qiang Qian $^{1}$,Songcan Chen$^{1,2}$\footnote{Corresponding author: Tel: +86-25-84896481 Ext. 12221; Fax: +86-25-84892400; E-mail: s.chen@nuaa.edu.cn(S. Chen) qian.qiang.yx@gmail.com(Q. Qian)} }

\begin{document}
\maketitle
1. Department of Computer Science and Engineering, Nanjing University of Aeronautics and Astronautics, Nanjing 210016, P. R. China \\ \newline
2. State Key Laboratory for Novel Software Technology, Nanjing University, Nanjing 210093, P. R. China
\section{Abstract}
In this paper, we attempts to learn a single metric across two heterogeneous domains where source domain is fully labeled and has many samples while target domain has only a few labeled samples but abundant unlabeled samples. 
To the best of our knowledge, this task is seldom touched. 
The proposed learning model has a simple underlying motivation: all the samples in both the source and the target domains are mapped into a common space, where both their priors $P(sample)$s and their posteriors $P(label|sample)$s are forced to be respectively aligned as much as possible.
We show that the two mappings, from both the source domain and the target domain to the common space, can be reparameterized into a single positive semi-definite(PSD) matrix.
Then we develop an efficient Bregman Projection algorithm to optimize the PDS matrix over which a LogDet function is used to regularize. 
Furthermore, we also show that this model can be easily kernelized and verify its effectiveness in cross-language retrieval task and cross-domain object recognition task.

\newpage
\section{Introduction}
\label{sec:introduction}
Metric learning lies in the heart of many machine learning tasks such as clustering and recognition, thus has been extensively studied by many researchers. 
However, most of the works only focus on learning metric for a single domain\cite{liu2010joint,xing2002distance,bian2012constrained,wang2011semisupervised,liu2010joint,weinberger2006distance}, leaving metric learning across multiple-domains seldom touched.
In this paper, we introduce the Metric Learning across Heterogeneous Domains(MLHD) model to learn a single metric across two heterogeneous domains, which means not only their sample distributions but also their feature spaces are different.
Between the two domains, the source domain, which has been collected beforehand, are fully labeled and has many samples. 
While the target domain has only a few labeled samples but has abundant unlabeled samples because collecting labels is expensive and tedious. 
Since the samples in the two domains may disagree on their feature dimensions, the metrics between them can not be calculated directly. 
A simple and direct idea, as depicted in Figure \ref{fig:MappingCommonSpace} is to linearly map the samples in both the domains into a common space where their metrics can be calculated.
And at the same time, the two mappers, from both the source domain and the target domain to the common space, should satisfy some constraints:
First, samples sharing the same labels in both the domains should be close to each other in the common space.
In other words, posterior $P(label|sample)$s of both the domains in the common space should be aligned as closely as possible.
As demonstrated in Figure \ref{fig:subfig:AlignPosterior}, the red circles and squares are close to each other, and so are the blue circles and squares in that common space.
However, this is not enough because the target domain only owns a few labeled samples. 
And if the posteriors are aligned only based on such a small portion of labeled samples in the target domain, they can be likely biased and lead to poor generalization.
Figure \ref{fig:subfig:AlignPosterior} shows that the gray unlabeled samples are aligned poorly though the color labeled samples are aligned well.
To alleviate this problem, we also need to force the priors $P(sample)$s of both the domains are aligned as much as possible.
Since the target domain usually contains many unlabeled samples, the estimation of its prior is relatively more reliable. 
And by aligning the priors, the probable bias introduced by the poor posterior of the target domain can be corrected to some extent. 
Figure \ref{fig:subfig:AlignPriorPosterior} shows that both the unlabeled and the labeled samples are well located in the common space by respectively aligning the priors and the posteriors. 

\begin{figure}[H]
\centering
\subfigure[Aligning posteriors]{
\includegraphics[scale=0.2]{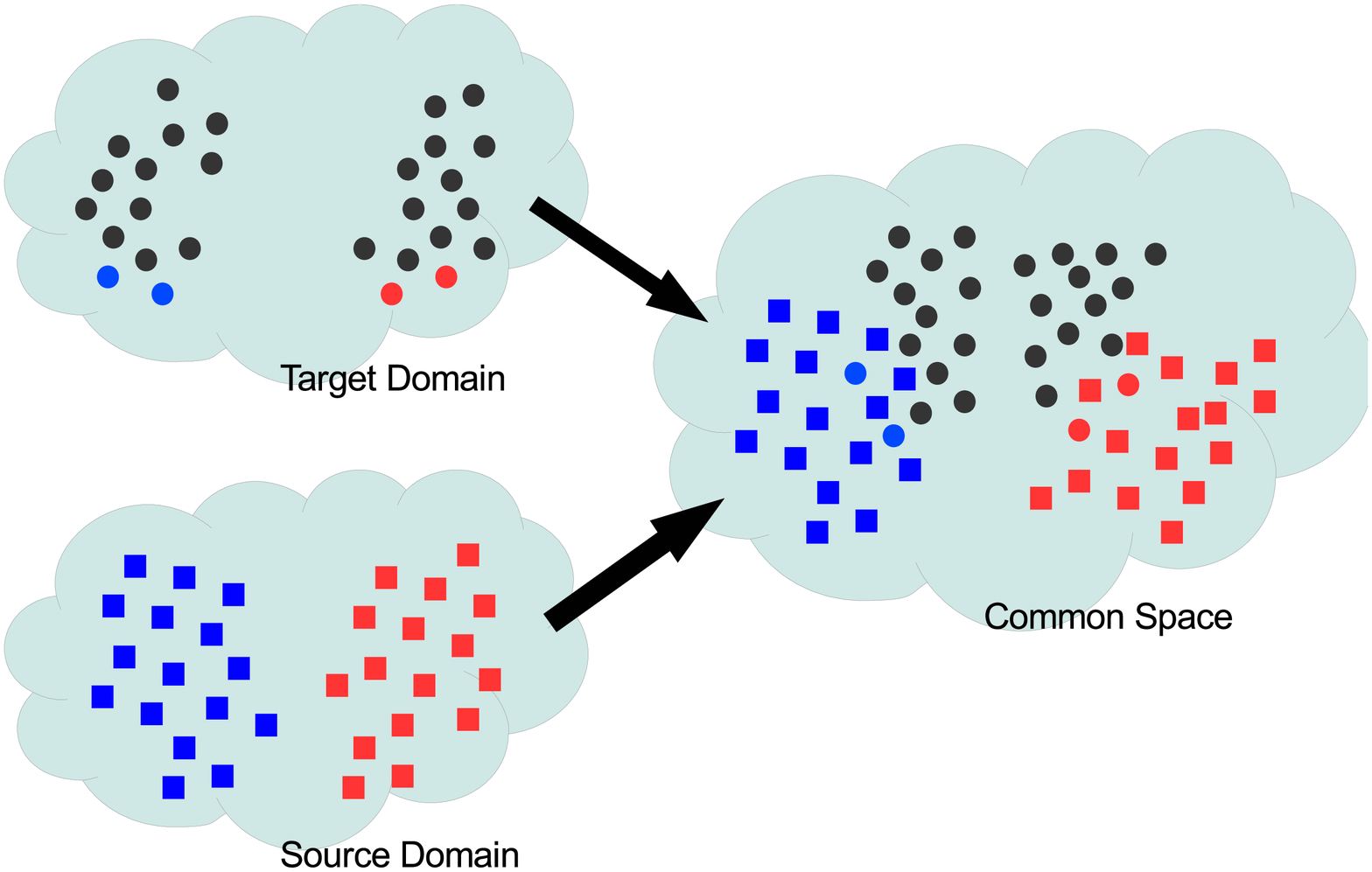}
\label{fig:subfig:AlignPosterior}
}
\subfigure[Aligning priors and posteriors]{
\includegraphics[scale=0.2]{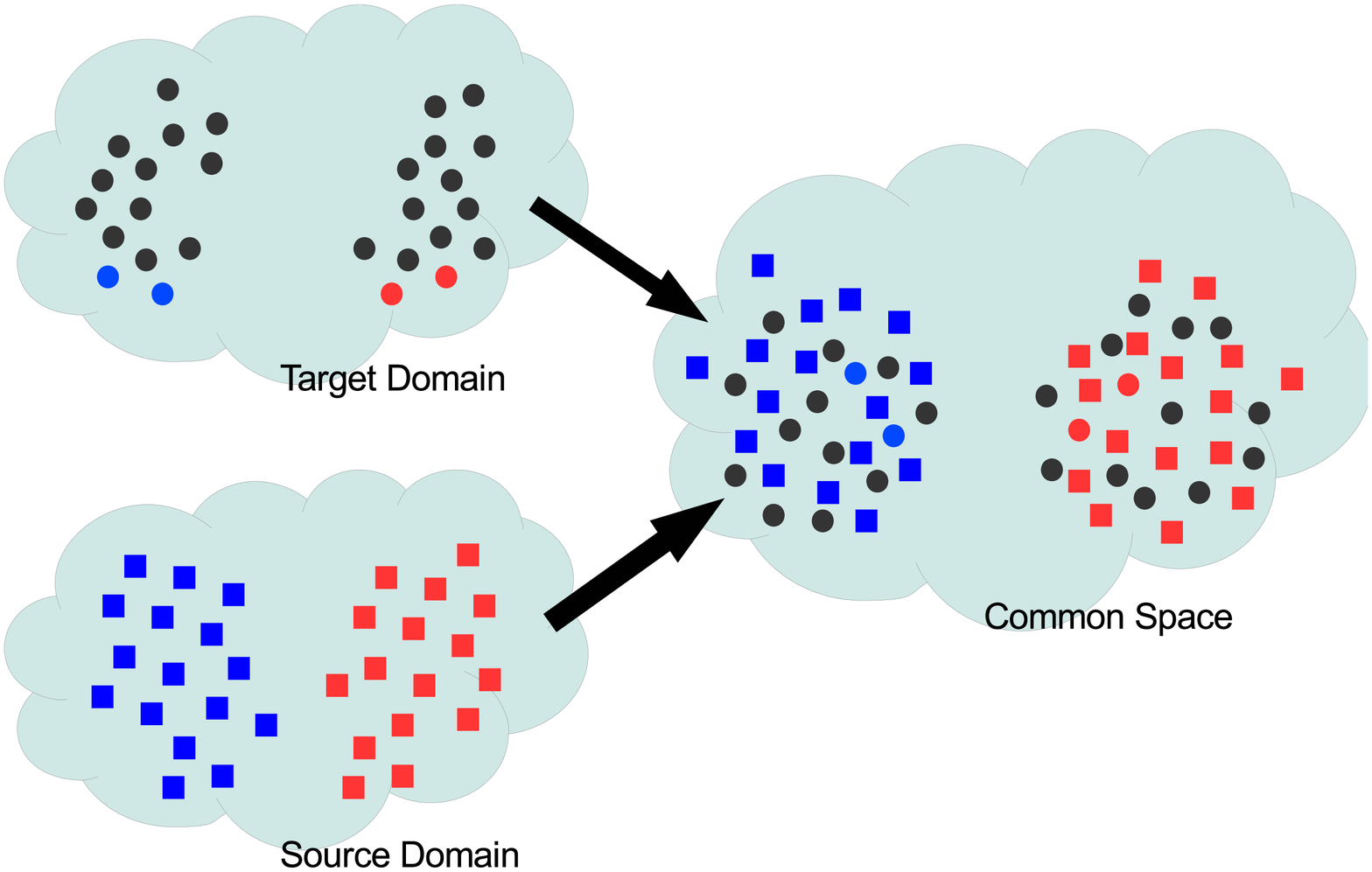}
\label{fig:subfig:AlignPriorPosterior}
}
\caption{Mapping the samples of the source and the target domains to a common space. Circles and squares are the samples in the target and the source domains respectively. Colors indicate the labels of the samples. Gray represents the unlabeled samples. }
\label{fig:MappingCommonSpace}
\end{figure}

One advantage of our model, deserved to be highlighted, is that the learned metric is actually defined in the common space. 
That means we can calculate the metric across the source domain and the target domain. 
This advantage brings much convenience for some practical applications.
For example, this metric can directly serve the retrieval tasks cross heterogeneous domains such as cross-language retrieval and cross-domain object recognition.
Although there are a few works about metric learning on heterogeneous domains,  some of them focus on learning metric only in the target domain\cite{qi2012transfer}, and only rare focus on learning metric(or similarity) across heterogeneous domains.
Kate et al.\cite{saenko2010adapting} proposed to learn a metric across different image domains, however their model is limited to the situation where the dimensions of both the source and the target domains are the same. 
Kulis et al.\cite{kulis2011you} broke this limit, however their model learns a similarity function which only calculates the cross-domain similarity.
Moreover, both of the above works just align the posteriors from this paper's perspective, thus do not exploit the unlabeled samples of the target domain, which generally could be very useful.

The formulation of our MLHD model consists of three parts:
the first two parts force the alignments of the priors and the posteriors between the two domains. 
The priors are aligned by minimizing a two-sample testing statistic, Maximum Mean Discrepancy(MMD), proposed by Gretton et al.\cite{gretton2008kernel}.
And the posteriors are aligned by keeping the samples in the same class close enough while in the different classes as far as possible.
We show that the two parts can be reparameterized with a single PSD matrix.
Then we introduce a LogDet regularizer over the PSD matrix as the third part to avoid the troublesome PSD constraint because it can be automatically satisfied\cite{davis2007information,kulis2006learning} by using Bregman Projection algorithm if the objective is LogDet function. 
Besides, to deal with the nonlinear situation, we also kernelized our model based on the kernelization works about LogDet function\cite{jain2010inductive,jain2009metric}.
Our paper provides the detailed proofs on kernelization.

The rest of the paper is organized as follows: Section \ref{sec:related} reviews some related works, and section \ref{sec:formulation} presents the MLHD model in detail. Then Section \ref{sec:experiments} verifies the effectiveness of MLHD experimentally. Finally section \ref{sec:conclusion} conclude this paper.

\section{Related Work}
\label{sec:related}
Learning among heterogeneous data sources has caught much researchers' attention. 
However, most of them focus on the classification or dimension reduction learning in the target domain with the help of the source domains. 
Only rare of them focus on the metric learning across the heterogeneous domains. 
Thus in this section, besides of the related works on metric learning across the heterogeneous domains, we also review some recent heterogeneous classification learning works at first.


Heterogeneous learning may date back to Dai et al.\cite{dai2008translated}.
They used some co-occurrence data to estimate the feature-level conditional distribution from source feature to target feature. 
Later, many other methods were proposed\cite{zhang2011multi,wang2009heterogeneous,wang2011heterogeneous,shi2010transfer,duan2012learning}. 
A common character of these methods is that they all map the samples in the source and the target domains into a common space for the learning tasks.
For example, Wang et al.\cite{wang2011heterogeneous} embedded all the samples in different domains into a common space according to a large manifold structure covering both the within-domain geometrical structure and the between-domain label structure.
Zhang et al.\cite{zhang2011multi} mapped all the samples into a common space and applied the classic linear discriminant analysis(LDA). 
Shi et al.\cite{shi2010transfer} used a collective matrix factorization model to find out the common space.
However, the algorithm requires the same number of samples of source and target domains, which is usually could not be satisfied. Thus before conducting the algorithm, they had to bring in a sampling procedure.
Duan et al.\cite{duan2012learning} constructed a parameterized augmented space as the common space motivated by a domain adaptation method proposed by Daume et al. did\cite{daume2007frustratingly}.
And the parameters are learned through optimizing a large margin classification model. 

Although learning among heterogeneous data sources has attracted much attention, works on metric learning across heterogeneous domains are relatively rare. 
Qi et al.\cite{qi2012transfer} focused on metric learning only for the target domain, but not that across the source and the target domains, thus concern different setting from ours.
To the best of our knowledge, Kulis et al.\cite{kulis2011you}'s work is the only one closest to ours, although what they learned is, strictly speaking, a similarity function rather than a metric across the source and the target domains.
They proposed a Frobenuis-norm regularized large margin model to learn the (linear) similarity function, which, from this paper's perspective, can be seen as only aligning the posteriors rather than the priors.
Thus, they don't explore the abundant available unlabeled samples in target domain to leverage the learning.

\section{Metric Learning across Heterogeneous Domains}
\label{sec:formulation}
In this section, we first present some notations, then give out the mathematical model. Next, we optimize the model with Bregman projection method. Finally, we show how to kernelize this model. 
\subsection{Problem Statement and Notations}
We first provide some notations used throughout this paper.
Assuming that we are given two domains: a $D^x$ dimensional labeled source domain $\mathcal{X}=\{ (x_i,l^x_i)|i=1,2,\cdots,N^x \}$, and a $D^y$ dimensional partially labeled target domain $\mathcal{Y}=\{(y_i,l^y_i)|i=1,2,\cdots,N^y_l\} \cup  \{y_i|i=N^y_l+1,\cdots,N^y_l+N^y_u\}$.
Let $N^y=N^y_l+N^y_u$.
For convenience, we also define two data matrix $X=[x_1,x_2,\cdots,x_{N^x}] \in \mathcal{R}^{D^x \times N^x}$ and $Y=[y_1,y_2,\cdots,y_{N^y}] \in \mathcal{R}^{D^y \times N^y}$.
Then two linear operators $W_x\in \mathcal{R}^{D^x \times D^c}$ and $W_y\in \mathcal{R}^{D^y \times D^c}$ are used to map the samples in the two domains into a $D^c$ dimensional common space. 
Specifically $x \rightarrow W_x^Tx$ and $y \rightarrow W_y^Ty$.
And the metric is defined as the 2-norm $d(x,y) = \|W_x^Tx - W_y^Ty\|_2$. 
Furthermore the squared metric can be rewritten into a matrix form as follows:
\begin{align*}
 d^2(x_i,y_j) &= \| W_x^Tx_i - W_y^Ty_j \|^2   \\
             &= [x_i^T ~~ -y_j^T]  
\left( 
\begin{array}{cc}
W_xW_x^T &  W_xW_y^T \\
W_yW_x^T &  W_yW_y^T \\
\end{array}
\right)
\left[
\begin{array}{c}
 x_i \\
 -y_j
\end{array}
\right]    \\
\end{align*}
If we let 
\begin{equation*}
M = \left( 
\begin{array}{cc}
W_xW_x^T &  W_xW_y^T \\
W_yW_x^T &  W_yW_y^T \\
\end{array}
\right)
\end{equation*}
and $z_{ij} = [x_i^T ~~ -y_j^T]^T$, then we have 
\begin{equation}
d^2_M(x_i,y_j) = z_{ij}^T M z_{ij}
\label{eq:metric}
\end{equation}
which is reparameterized only by matrix $M \in \mathcal{S}_+$, where $\mathcal{S}_+$ denotes the set containing all the symmetric positive semi-define matrices. 

The goal of metric learning across the heterogeneous domain is to learn the parameterized metric $d$ defined above by using the data in both the source domain $\mathcal{X}$ and the target domain $\mathcal{Y}$.

\subsection{Formulation}
In this subsection, we propose our  Metric Learning across Heterogeneous Domain(MLHD) model, which fully exploits both the labeled and the unlabeled samples in the two domains.
And to reach this goal, we force the model to align not only the posteriors but also the priors of the two domains.

Aligning the posteriors amounts to forcing the samples in the same class close enough while the samples in the different classes far away. 
And it is easy to achieve by imposing the following distance constraints:
\begin{align}
d^2(x_i,y_j) &\ge l  \quad  \text{if } l^x_i = l^y_j \\
d^2(x_i,y_j) &\le u  \quad  \text{if } l^x_i \neq l^y_j 
\end{align}

Aligning the priors is closely related to the statistical two-sample testing problem, which determines whether two random variables have the same distribution. 
So first of all, let us briefly introduce the method, proposed by Gretton et al.\cite{gretton2008kernel}, for the two-sample testing problem.
In that paper, the authors used a kernel method to judge the discrepancy between two random variables. 
And the proposed statistic, named Maximum Mean Discrepancy(MMD), calculates the distance between the means of the two random variables mapped into a Reproducing Kernel Hilbert Space.
Then the authors presented some critical statistical analytic results.
The first one is that if the kernel is universal\cite{micchelli2006universal}, then MMD=0 if and only if the two random variables are the same.
The authors also showed that the empirical MMD converges in probability at rate $1/\sqrt{\text{total number of the samples}}$.
In this paper, we align the priors of the two domains in the common space by minimizing the squared MMD statistic on the samples mapped in the common spaces. 
The formulation is as follows
\begin{equation}
\begin{array}{l}
MMD^2(\mathcal{X},\mathcal{Y}) =  \\
\left( \frac{1}{N_x^2}\sum_{i,j=1}^{N_x}k(W_x^Tx_i,W_x^Tx_j)  + \frac{1}{N_y^2}\sum_{i,j=1}^{N_y}k(W_y^Ty_i,W_y^Ty_j) -  \frac{2}{N_xN_y}\sum_{i,j=1}^{N_x,N_y}k(W_x^Tx_i,W_y^Ty_j)     \right)
\end{array}
\label{eq:mmdorg}
\end{equation}
where $k(\cdot,\cdot)$ is an universal kernel function, for example, Gaussian kernel. 
Unfortunately, minimizing the equation \ref{eq:mmdorg} with respect to $W_x$ and $W_y$ is difficult due to 1)  it is nonconvex and 2) $W_x$ and $W_y$ are embedded in kernel which is nonlinear.
Thus to make the issue tractable, instead we just use the linear kernel to make equation \ref{eq:mmdorg} convex and reparameterize it using matrix $M\in\mathcal{S}^+$. For this, let 
\begin{equation}
\overline{z} = \left[ \frac{1}{N_x}\sum_{i=1}^{N_x} x_i^T ~~ -\frac{1}{N_y}\sum_{j=1}^{N_y} y_j^T \right]^T
\end{equation}
The squared MMD now can be simplified to:
\begin{align}
MMD^2_M(\mathcal{X},\mathcal{Y})   =  \overline{z}^T M \overline{z}
\label{eq:mmd}
\end{align}

So far, the alignments of the priors and the posteriors both simply depend on the PSD matrix $M$, however such PSD constraint is relatively troublesome for optimization. 
Fortunately, LogDet-function regularized model can automatically keep the PSD property of the $M$ in the optimization process while still hold the convexity\cite{kulis2009low}.
Consequently we use the LogDet function to regularize the matrix $M$ as follows: 
\begin{equation}
LogDet(M,M_0) = tr(MM_0^{-1}) - \log \det(MM_0^{-1}) - dim(M)  
\end{equation}
where $tr(\cdot)$ is the trace operator, and $dim(\cdot)$ is the dimension function.

Now by fusing the LogDet regularizer,  we recast our MLHD model as follows
\begin{equation}
\begin{array}{rl}
\min_{M,t,\bm\xi} ~~ &LogDet(M,I) + \lambda_1 MMD^2_M(\mathcal{X},\mathcal{Y})  + \lambda_2 LogDet(\bm\xi,\bm\xi_0)\\
s.t.~~ &d^2_M(x_i,y_j) \ge \xi_{ij}  \quad  \text{if } l^x_i = l^y_j \\
       &d^2_M(x_i,y_j) \le \xi_{ij}  \quad  \text{if } l^x_i \neq l^y_j  \\
       & M \in \mathcal{S}_+
\end{array}
\label{eq:model}
\end{equation}
where $\bm\xi$ is a vector of slack variables, and $\bm\xi_0$ is an initializing vector whose components equal $u$ if corresponding to same-class constraints and $l$ if corresponding to different-class constraints.
$I$ is the identity matrix.
And $\lambda_1$ and $\lambda_2$ are two trade-off parameters. 
This is still a convex model.

\subsection{Optimization}
In this section, we use Bregman Projection algorithm\cite{censor1997parallel,bregman1967relaxation} to optimize our model. 
The algorithm cyclically projects the current solution onto a single constraint with  Bregman divergence, here the LogDet function. 
To facilitate the projection, we first relax the equation \ref{eq:model} and make its objective function only contain LogDet function as follows:
\begin{equation}
\begin{array}{rl}
\min_{M,t,\bm\xi} ~~ &LogDet(M,I) + \lambda_1 LogDet(t,t_0)  + \lambda_2 LogDet(\bm\xi,\bm\xi_0)\\
s.t.~~ &d^2_M(x_i,y_j) \ge \xi_{ij}  \quad  \text{if } l^x_i = l^y_j \\
       &d^2_M(x_i,y_j) \le \xi_{ij}  \quad  \text{if } l^x_i \neq l^y_j  \\
       &MMD^2_M(\mathcal{X},\mathcal{Y}) \le t \\
       & M \in \mathcal{S}_+
\end{array}
\label{eq:relaxion}
\end{equation}
where $t_0$ is small positive number as the initialization of $t$. 
Note that $t \ge 0$ is implied by constraint $MMD^2(\mathcal{X},\mathcal{Y}) \le t$,  thus $t$ can be placed in the LogDet function. 

Then we present the optimization method described in algorithm \ref{alg:algorithm}. 
The algorithm also cyclically projects the current solution onto a single linear constraint with LogDet function, consequently these projections can be analytically solved. 
Due to the LogDet function's property that it is only defined over PSD matrix set, the projected result is still restricted in $\mathcal{S}^+$. 
In fact, similar methods are also used in \cite{davis2007information,kulis2006learning}. 
\begin{algorithm}
\begin{algorithmic}
  \STATE \textbf{Input:} Source domain $X$, target domain $Y$, parameter $\lambda_1$,$\lambda_2$
  \STATE \textbf{Initialize:} primer variables $M=I$,$\bm\xi=\bm\xi_0$,$t=t_0$, dual variables $\beta=0$,$\zeta=0$
  \WHILE { $n < MaxIter$ }
  \STATE -------- Bregman Projection on distance constraints ------------------------
     \FOR {Each distance constraint}
     \STATE 1: Solving the following problem by Lagrangian method, and getting its Lagrangian multiplier $\alpha$
           \begin{align*}
             \min_{M,\xi_{ij}} ~~ &  LogDet(M,M^n_{ij}) + \lambda_1 LogDet(\xi_{ij},\xi_{ij}^n) \\
             s.t.      ~~ &  z_{ij}^T M z_{ij} = \xi_{ij}
           \end{align*}         
     \STATE 2: Update primer variables $M,\xi_{ij}$ and dual variable $\beta_{ij}$  \\
          \center{If $l^x_i=l^y_j$, then $\delta=1$, else $\delta=-1$.}
          \begin{align*}
            &p = z_{ij}^T M_{ij}^n z_{ij} \\
            &\alpha = \min(\beta_{ij},\frac{\delta\lambda_2}{1+\lambda_2} ( 1/p - 1/\xi_{ij}^n))\\
            &\beta_{ij} = \beta_{ij} - \alpha \\
            &\xi_{ij} = \lambda_2 \xi_{ij}^n  / (\lambda_2 + \delta \alpha \xi_{ij}^n) \\
            &M = M_{ij}^n + \frac{\lambda_2\xi_{ij}^n}{\lambda_2+\delta \alpha \xi_{ij}^n} M_{ij}^n z_{ij}z_{ij}^TM_{ij}^n
          \end{align*}
     \ENDFOR
  \STATE -------- Bregman Projection on MMD constraint ------------------------------
     \STATE 1: Solving the following problem by Lagrangian method, and getting its Lagrangian multiplier $\eta$
           \begin{align*}
             \min_{M,t} ~~ &  LogDet(M,M^n) + \lambda_1 LogDet(t,t^n) \\
             s.t.      ~~ &  \overline{z}^T M \overline{z} = t 
           \end{align*}
     \STATE 2: Update primer variables $M,t$ and dual variable $\zeta$
           \begin{align*}
	     &\eta = - min(\zeta, -\eta)  \\
             &\zeta = \zeta + min(\zeta,-\eta)  \\  
             &M = M^n -  \frac{1}{\overline{z}^T M^n \overline{z} - \eta}  M^n \overline{z} \overline{z}^T M^n \\
             &t = \frac{t^n \lambda_1}{t^n \eta + \lambda_1}
           \end{align*} 
  \ENDWHILE
\end{algorithmic}
\caption{Optimization algorithm for MLHD model}
\label{alg:algorithm}
\end{algorithm}

\subsection{Kernelization}
The MLHD model established in equation \ref{eq:model} is linear, thus it is inappropriate for nonlinear circumstances.
As a widely accepted solution, kernel method, through nonlinearly mapping the original samples into a high-dimensional space and conducting learning in that space, can conveniently convert a linear model into a nonlinear model\cite{smolalearning}.  
In this subsection, we present how to kernelize the above linear model.
We first introduce some notations. 
Denote $Q$ by
\begin{equation}
Q = \left[  \begin{array}{cc}
            X &  \\
	    & Y  \\
	    \end{array}  \right]   \in \mathcal{R}^{(D_x+D_y)\times(N_x+N_y)} 
\end{equation}
Denote the kernel function defined on $\mathcal{X},\mathcal{Y}$ by $k_x(\cdot,\cdot)$ and $k_y(\cdot,\cdot)$ respectively.
Let $K_x,K_y$ be the kernel matrix with the $(i,j)^{th}$ entry be $k_x(x_i,x_j)$ and $k_y(y_i,y_j)$ respectively and the kernel matrix $K$ be 
\begin{equation}
K = \left[  \begin{array}{cc}
            K_X &  \\
	    & K_Y  \\
	    \end{array}  \right]  \in \mathcal{R}^{(N_x+N_y)\times(N_x+N_y)} 
\end{equation}
Let $e_{ij} = [ e_i^T ~~ -e_j^T ]^T$ where $e_k$ is a vector with only the $i$th entry being 1.
Then the squared metric in equation \ref{eq:metric} can be cast as $d^2(x_i,y_j) = e_{ij}^TQ^TMQe_{ij}$.
Let $\overline{e} = [ 1_{N_x}/N_x ~~ -1_{N_y}/N_y ]$ where $1_N$ is the $N$ dimensional vector with all entries be 1.
Then the squared MMD in equation \ref{eq:mmd} can be rewritten as $MMD^2(\mathcal{X},\mathcal{Y}) = \overline{e}^TQ^TMQ\overline{e}$.

We follows the idea in \cite{kulis2011you} to kernelize our MLHD model. 
Specifically, we first show that the range space of the matrix parameter $M$ in equation \ref{eq:model} is in the range space of $Q$, then derive an equivalent optimization problem which only depends on the inner product defined in the source and the target domains. 
Finally, the inner product can be substituted with any kernel function. 
The concrete kernelization is summarized in the following theorems \ref{theo:range} and \ref{theo:equivalent}.
Although our kernelization looks like that in \cite{kulis2011you}, there are some difference: 1) our model focuses on metric thus its parameter matrix $M$ is PSD matrix while the parameter matrix in \cite{kulis2011you} is asymmetric rectangle matrix, 2) the regularizer used in our model is LogDet function while is Frobenius norm in \cite{kulis2011you}.

In the following, we shown in theorem \ref{theo:range} that the range space of the matrix parameter $M$ in equation \ref{eq:model} is in the range space of $Q$.
\begin{theorem}
There exists an  $N_x+N_y$ dimensional matrix $L \in \mathcal{S}_+$ such that the optimal solution $M^\star$ to \ref{eq:model} is of the form as follows
\begin{equation}
M^\star = QK^{-1/2}LK^{-1/2}Q^T
\end{equation}
\label{theo:range}
\end{theorem}
\begin{proof}
Apparently, $M\in\mathcal{S}_+$, since LogDet only defines on $\mathcal{S}_+$.
Let $Q_\perp$ consists of the basis vectors spanning the null space of $Q^T$, i.e., $Q^TQ_\perp=0$. 
Then $M$ can be decomposed into two parts as follows
\begin{equation} 
M=Q\tilde{L}Q^T + Q_\perp \tilde{L} Q_\perp^T
\end{equation}
where $\tilde{L}$ is some PSD matrix.
It is easy to show that the second term $Q_\perp \tilde{L} Q_\perp^T$ has no influence on $d^2(x_i,y_j)$ and $MMD^2(\mathcal{X},\mathcal{Y})$. 
Consequently the only term in equation \ref{eq:model} influenced by the second term of $Q$ is the LogDet term.
Fortunately, the LogDet term is only determined by the eigenvalues of $M$ 
\begin{equation}
LogDet(M,I) = \sum_i \sigma_i(M) + \sum_i \log \sigma_i(M) - dim(M)
\end{equation}
where $\sigma_i(\cdot)$ is the $i$th largest eigenvalue.
And according to matrix perturbation theory\cite{stewart1990matrix}, $\sigma_i(Q\tilde{L}Q^T) \le \sigma_i(Q\tilde{L}Q^T + Q_\perp \tilde{L} Q_\perp^T)$ when both $Q\tilde{L}Q^T$ and $Q_\perp \tilde{L} Q_\perp^T$ are PSD matrices. 
Thus, to minimize the objective in equation \ref{eq:model}, we should discard $Q_\perp \tilde{L} Q_\perp^T$ term and let $M^\star=Q\tilde{L}Q^T$. 
Finally by transforming $L=K^{1/2}\tilde{L}K^{1/2}$, we can write $M^\star = QK^{-1/2}LK^{-1/2}Q^T$. \qed
\end{proof}

Then based on the above theorem \ref{theo:range}, we show in the following theorem \ref{theo:equivalent} that an equivalent optimization problem can be derived and only involves the inner products defined in the source and the target domains.

\begin{theorem}
If $L^\star$ is the optimal solution to the following problem:
\begin{equation}
\begin{array}{rl}
  \min_{L,t,\bm\xi} ~~ & LogDet(L,I) +  \lambda_1 LogDet(t,t^0) + \lambda_2 LogDet(\bm\xi,\bm\xi^0) \\
  s.t.           ~~ & e_{ij}^T K^{1/2} L K^{1/2} e_{ij} \ge \xi_{ij}  \quad  \text{if } l^x_i = l^y_j \\
                    & e_{ij}^T K^{1/2} L K^{1/2} e_{ij} \le \xi_{ij}  \quad  \text{if } l^x_i \neq l^y_j \\
                    & \overline{e}^TK^{1/2} L K^{1/2} \overline{e} \le t \\
\end{array}
\label{eq:L}
\end{equation}
then $M^\star = QK^{-1/2}L^\star K^{-1/2}Q^T$.
\label{theo:equivalent}
\end{theorem}
\begin{proof}
Note that 
\begin{equation}
QK^{-1/2} = 
\left[ 
\begin{array}{cc}
XK_X^{-1/2} &   \\ 
& YK_Y^{-1/2}  \\
\end{array}
\right]
\end{equation}
is an orthogonal matrix because both $XK_X^{-1/2}$ and $YK_Y^{-1/2}$ are orthogonal matrices. 
Let the eigen-decomposition of $L$ be $L=U\Sigma U^T$. 
Then $M=(QK^{-1/2}U) \Sigma (U^TK^{-1/2}Q^T)$, which is the eigen-decomposition of $M$.
Consequently $\sigma_i(M) = \sigma_i(L)$, which means 
\begin{equation}
LogDet(M,I) = LogDet(L,I) + const 
\label{eq:L1}
\end{equation}

Also by substituting $M=QK^{-1/2}L K^{-1/2}Q^T$ into equations \ref{eq:metric} and \ref{eq:mmd}, we have 
\begin{equation}
d^2(x_i,y_j) = e_{ij}^T K^{1/2} L K^{1/2} e_{ij}
\label{eq:L2}
\end{equation}
\begin{equation}
MMD^2(\mathcal{X},\mathcal{Y}) = \overline{e}^TK^{1/2} L K^{1/2} \overline{e}
\label{eq:L3}
\end{equation}
By rewriting the equation \ref{eq:model} using equations \ref{eq:L1}, \ref{eq:L2} and \ref{eq:L3}, we have the equivalent optimization problem \ref{eq:L}, and also have $M^\star = QK^{-1/2}L^\star K^{-1/2}Q^T$. \qed
\end{proof}

\newpage
\section{Experiments}
\label{sec:experiments}
This section verifies the MLHD model experimentally through the comparison with some relevant baselines.
We first visually demonstrate the idea of the MLHD model under two-dimensional source domain and three-dimensional target domain. 
Then we conduct two experiments on the cross-language retrieval task and the heterogeneous domain object recognition task respectively. 

\subsection{Baseline Methods}

\begin{itemize}
\item KCCA+NN \\
Because of the different dimensional spaces where the source and the target domains lie, nearest neighbor(NN) classifier cannot be applied directly. 
Consequently, we follow the methods used by Kulis et al\cite{kulis2011you}. Specifically, we first apply kernel canonical correlation analysis(KCCA) to project the samples in the two domains into a common space, then run the NN classifier. 
\item KCCA+ITML \\
Using information theory metric learning(ITML) algorithm\cite{davis2007information} to adapt the discrepancy between different domains is proposed by Saenko et al.\cite{saenko2010adapting}
However ITML cannot work in different dimensional spaces, thus KCCA is first applied.
This baseline is also from Kulis et al.'s paper\cite{kulis2011you}.
\item Asymmetric Regularized Cross-domain transformation(ARC) \\
This method was introduced by Kulis et al.\cite{kulis2011you}.
It learns a linear asymmetric transformation to compute the cross-domain similarity score, and its mathematical formulation is as follows
\begin{equation}
\min_{W} ~~ \|W\|_F^2 + \lambda \left(  \sum_{l^x_i=l^y_j} \max(0,l-x_i^TWy_j)^2  + \sum_{l^x_i\neq l^y_j} \max(0,x_i^TWy_j-u)^2 \right)
\end{equation}
The differences between ARC and MLHD are that, 1) it learns the similarity instead of the metric, 2) it does not consider the alignment of priors, thus fails to exploit the abundant unlabeled samples.
\end{itemize}

\subsection{Toy Problem} 
To demonstrate the benefit of respectively aligning both priors and posteriors of the two domains, we construct a two-dimensional source domain and a three-dimensional target domain, as depicted in figure \ref{fig:demo_data}. 
Both of them have two classes. 
And in the source domain, each class has 40 labeled samples randomly drawn from two Gaussian distributions.
While in the target domain, to demonstrate the efficacy of aligning prior, each class has 40 unlabeled samples and 2 labeled samples, deliberately sampled with bias. 
\begin{figure}[H]
\centering
\includegraphics[scale=0.4]{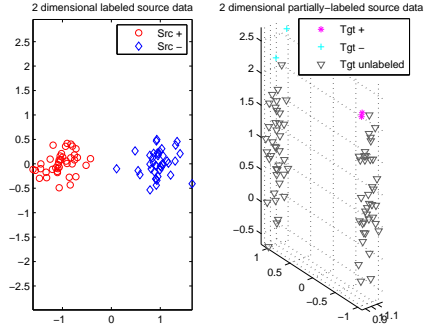}
\caption{Source and target domains}
\label{fig:demo_data}
\end{figure}
Figure \ref{fig:demo_results} shows the results of the three baselines and our algorithm. 
Clearly, CCA does not map well into common space. 
The samples in each class from both domains lie mainly in a strip, and the unlabeled samples totally are separated from the labeled ones. 
Predictably, the nearest neighbor classifier will report a bad accuracy under this situation. 
ITML makes the situation better, but still fails to align the two domains well. 
The unlabeled samples are still separated from the labeled ones.
ARC only utilizes the labeled samples of the two domains. 
Although the two classes, of either the source or the target domain, are separated well, the distributions of the source and the target domains are obviously very different. 
On the contrary, besides aligning the posteriors, MLHD explicitly forces the prior distributions to be aligned. 
And the figure shows that the two classes, of both the source and the target domains, align withtogether. 
Note that, the classes of the target domain roughly concentrate in the center of the corresponding classes of the source domain. 
They does not align evenly, because in MLHD model, linear kernel is used rather than the required universal kernel in the MMD term for the optimization convenience. 
As a results, only the means of the priors are aligned, not the priors themselves. 

\begin{figure}[H]
\centering
\includegraphics[scale=0.4]{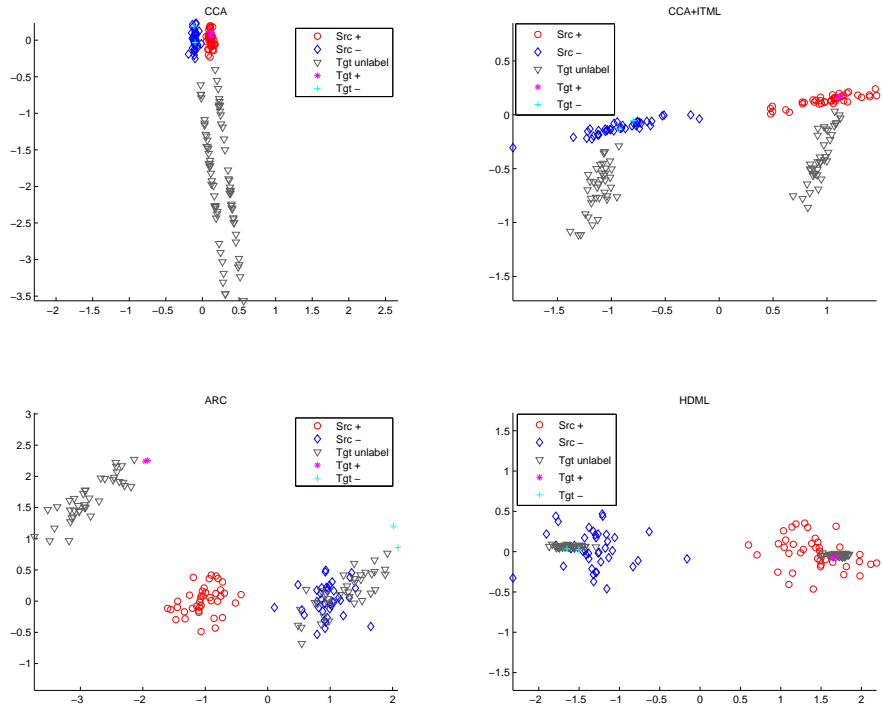}
\caption{Visualization on the toy problem}
\label{fig:demo_results}
\end{figure}

\subsection{Experiments on multilingual Reuters dataset}
The dataset of multilingual Reuters is collected by Amini et al.\cite{amini2010learning}\footnote{http://multilingreuters.iit.nrc.ca/ReutersMultiLingualMultiView.htm}. 
It contains 6 large Reuters categories(CCAT, C15, ECAT, E21, GCAT and M11) extracted from 5 different languages(English, French, German, Italian and Spanish) collections.
Each document has been preprocessed and indexed using a standard preprocessing chain including removal of stopwords and low-frequency words, then is represented by TFIDF features. 
For convenience of computation, PCA is first applied to reduce the dimension of the source domain to 100, and the target domain to 150.

Twenty groups are constructed by picking any two languages as a group, one for source domain and the other for target domain.
For each group, ten trials of experiments are carried out, and the average accuracy as well as the standard deviation are reported.
And for each trial, 20 labeled samples per class are randomly chosen from the source domain, 20 unlabeled plus 1 labeled samples per class are randomly chosen from the target domain as the training set, and 300 samples are randomly chosen from the target domain as the testing set. 
In all of the trials, the RBF kernel is used. 

From the comparative results shown in table \ref{tab:multilang},we can see that the baseline KCCA+NN, which simply adopts the Euclidean distance, yields the poorest performance. 
The baseline KCCA+ITML is much better comparing with KCCA+NN, and implies that the common space produced by KCCA is not quite suitable.
ARC algorithm does not require that the samples in the two domains are in the space with the same dimensionality, thus does not need a preprocessing of KCCA for dimension reduction. 
And the results of ARC are generally better than those of KCCA+ITML. 
Compared with ARC, MLHD further aligns the priors, and outperforms the ARC overall, which confirms the benefit of aligning the priors. 
\begin{table}[H]
  \begin{center}
    \begin{tabular}{|l||r|r|r|l|}
      \hline
      Src-Tgt  &  KCCA+NN  &       KCCA+ITML  &            ARC  &  MLHD            \\
      \hline
      EN-FR    &    19.3$\pm$4.5  &  \textbf{46.4$\pm$3.7}  &           45.8$\pm$6.0  &  \textbf{46.6$\pm$4.0}  \\
      EN-GR    &    18.3$\pm$3.0  &           42.6$\pm$6.5  &           43.5$\pm$7.0  &  \textbf{45.0$\pm$4.0}  \\
      EN-IT    &    17.6$\pm$1.2  &           38.5$\pm$7.4  &  \textbf{40.0$\pm$6.0}  &  \textbf{40.0$\pm$4.9}  \\
      EN-SP    &    22.3$\pm$5.6  &           41.6$\pm$7.4  &           42.5$\pm$7.5  &  \textbf{44.1$\pm$5.6}  \\
      \hline
      FR-EN    &    20.1$\pm$3.5  &           36.6$\pm$8.3  &           45.1$\pm$6.8  &  \textbf{46.3$\pm$5.4}  \\
      FR-GR    &    17.6$\pm$2.1  &           36.3$\pm$8.1  &  \textbf{40.3$\pm$7.3}  &  \textbf{40.0$\pm$7.8}  \\
      FR-IT    &    18.2$\pm$2.4  &           32.5$\pm$6.3  &  \textbf{36.6$\pm$6.3}  &  35.8$\pm$6.1           \\
      FR-SP    &    17.8$\pm$2.7  &           34.7$\pm$6.6  &           37.6$\pm$8.3  &  \textbf{39.1$\pm$8.1}  \\
      \hline
      GR-EN    &    18.5$\pm$4.4  &           34.3$\pm$7.3  &           36.0$\pm$3.9  &  \textbf{38.2$\pm$4.9}  \\
      GR-FR    &    17.5$\pm$1.1  &           41.7$\pm$8.6  &  \textbf{44.1$\pm$6.9}  &  43.1$\pm$6.7           \\
      GR-IT    &    18.3$\pm$3.0  &           33.7$\pm$7.9  &           36.6$\pm$5.7  &  \textbf{37.7$\pm$6.7}  \\
      GR-SP    &    19.7$\pm$2.9  &           42.9$\pm$8.0  &           43.5$\pm$7.3  &  \textbf{45.7$\pm$6.4}  \\
      \hline
      IT-EN    &    19.0$\pm$4.4  &           36.3$\pm$6.9  &           39.8$\pm$3.8  &  \textbf{42.4$\pm$4.8}  \\
      IT-FR    &    18.7$\pm$3.0  &           37.6$\pm$8.3  &           41.2$\pm$7.0  &  \textbf{44.6$\pm$6.6}  \\
      IT-GR    &    20.5$\pm$7.2  &           39.9$\pm$8.8  &           43.0$\pm$8.2  &  \textbf{45.6$\pm$9.3}  \\
      IT-SP    &    17.2$\pm$0.6  &           36.1$\pm$7.8  &           37.6$\pm$6.2  &  \textbf{41.6$\pm$8.3}  \\
      \hline
      SP-EN    &    18.4$\pm$2.4  &           35.1$\pm$6.7  &  \textbf{42.3$\pm$8.4}  &  \textbf{42.8$\pm$9.3}  \\
      SP-FR    &    17.7$\pm$1.4  &           41.8$\pm$9.6  &           43.3$\pm$4.7  &  \textbf{44.8$\pm$6.2}  \\
      SP-GR    &    18.3$\pm$2.9  &           35.6$\pm$8.8  &  \textbf{43.9$\pm$8.2}  &  \textbf{43.9$\pm$7.6}  \\
      SP-IT    &    17.8$\pm$1.5  &           37.4$\pm$3.4  &           38.3$\pm$4.5  &  \textbf{40.0$\pm$6.4}  \\
      \hline
    \end{tabular}
  \end{center}
  \caption{Accuracy results of cross-language retrieval. The best performances are highlighted.}
  \label{tab:multilang}
\end{table}

The MLHD model has two important trade-off parameters: $\lambda_1$ for weighting the MMD term, and $\lambda_2$ for weighting the distance constraint term. 
To study how these parameters influence the performance of MLHD, we run the experiments with $\lambda_1$ taken from $[10^{-3} ~10^{-2} ~10^{-1} ~10^0 ~10^1 ~10^2]$ and $\lambda_2$ taken from $[10^{-2} ~10^{-1} ~10^0 ~10^1 ~10^2 ~10^3]$. 
The results are demonstrated in figure \ref{fig:param_multiling}. 
From this figure, we can observe that in general the MLHD is not very sensitive to the parameter configuration, especially to $\lambda_2$. 
Moreover, the accuracy increases as $\lambda_1$ increases. 
This also verifies the usefulness of aligning priors.

\begin{figure}[H]
\centering
\includegraphics[scale=0.4]{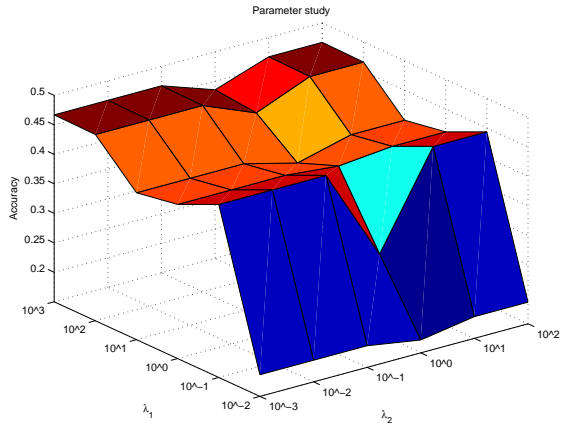}
\caption{Parameter study on Multilingual Reuters dataset}
\label{fig:param_multiling}
\end{figure}

\subsection{Object Recognition Experiments}
In this subsection, we use the dataset provided by Kate et al.\cite{saenko2010adapting}. 
This dataset contains 3 image domains: Amazon(images collected from Amazon.com), DSLR(high-resolution images taken from a digital DLR camera) and Webcam(low-resolution images taken from a web camera).
Among them, images in Amazon domains are in a canonical pose with uniform background, and images in both DSLR and Webcam domains are taken with varying poses and backgrounds.
Thus, in this experiments, we use DSLR and Webcam as the target domain separately. 
We follows the previous Kulis et al.\cite{kulis2011you}'s setting to extract image features. 
Specifically, all the images are first resized to 300x300 resolution. 
Then for each domains, three types of features are respectively extracted:
\begin{itemize}
\item \textbf{SURF600} SURF\cite{bay2006surf} features are extracted and clustered into a 600 visual words. Then each image is represented by a 600 dimensional histogram. 
\item \textbf{SURF800} Same processing as SURF600 besides clustering into a 800 visual words.
\item \textbf{SIFT900} SIFT\cite{lowe1999object} features are extracted and clustered into a 900 visual words. Then each image is represented by a 900 dimensional histogram.
\end{itemize}
We use the images with SURF600 and SIFT900 features as the source domain respectively and construct 16 groups of experiments in table \ref{tab:objrec}.
The experiment settings are almost the same as those in the above experiment on multilingual dataset. 
Specifically, ten trials are run for each group. 
And for each trial, 20 labeled samples per class are randomly chosen from the source domain. 
In the target domain, 10 unlabeled plus 1 labeled samples per class are randomly chosen as training set and the rest constitute the test set. 
We don't use many unlabeled samples in the target domain due to the limited number of collected images in DSLR and Webcam domains. 
In all of the trials, the RBF kernel is again used. 

The experimental results are listed in table \ref{tab:objrec}
\begin{table}[H]
  \begin{center}
    \begin{tabular}{|l||r|r|l|l|}
      \hline
      Src-Tgt                      &    CCA+NN  &  CCA+ITML  &  ARC                &  MLHD               \\
      \hline
      AmazonSurf600-WebcamSift900  &  15.7$\pm$2.6  &  29.7$\pm$2.6  &  30.1$\pm$3.2           &  \textbf{31.1$\pm$2.9}  \\
      AmazonSurf600-WebcamSurf800  &  19.0$\pm$3.9  &  29.0$\pm$3.1  &  \textbf{32.1$\pm$2.0}  &  30.6$\pm$2.6           \\
      AmazonSift900-WebcamSurf600  &  19.8$\pm$2.7  &  29.9$\pm$2.0  &  \textbf{31.9$\pm$2.8}  &  30.6$\pm$3.0           \\
      AmazonSift900-WebcamSurf800  &  18.7$\pm$4.4  &  30.1$\pm$4.6  &  \textbf{31.8$\pm$2.8}  &  \textbf{31.4$\pm$3.5}  \\
      DslrSurf600-WebcamSift900    &  16.1$\pm$1.9  &  29.1$\pm$1.8  &  29.8$\pm$1.9           &  \textbf{31.8$\pm$2.7}  \\
      DslrSurf600-WebcamSurf800    &  14.6$\pm$3.2  &  28.4$\pm$3.9  &  \textbf{31.4$\pm$3.2}  &  \textbf{31.6$\pm$2.7}  \\
      DslrSift900-WebcamSurf600    &  13.1$\pm$2.4  &  27.4$\pm$3.7  &  \textbf{29.9$\pm$1.9}  &  \textbf{29.4$\pm$1.3}  \\
      DslrSift900-WebcamSurf800    &  14.5$\pm$3.4  &  25.8$\pm$4.4  &  \textbf{29.3$\pm$2.5}  &  28.1$\pm$2.6           \\
      \hline
      AmazonSurf600-DslrSift900    &  20.6$\pm$5.6  &  28.8$\pm$5.1  &  27.8$\pm$4.5           &  \textbf{29.3$\pm$4.1}  \\
      AmazonSurf600-DslrSurf800    &  16.0$\pm$4.4  &  25.4$\pm$5.2  &  \textbf{27.8$\pm$4.1}  &  23.5$\pm$4.3           \\
      AmazonSift900-DslrSurf600    &  14.7$\pm$4.9  &  25.4$\pm$4.0  &  \textbf{26.5$\pm$3.8}  &  \textbf{26.5$\pm$4.6}  \\
      AmazonSift900-DslrSurf800    &  11.3$\pm$4.4  &  22.6$\pm$4.0  &  \textbf{25.5$\pm$4.4}  &  \textbf{25.5$\pm$4.6}  \\
      WebcamSurf600-DslrSift900    &  13.3$\pm$3.6  &  27.8$\pm$2.6  &  27.4$\pm$4.6           &  \textbf{29.8$\pm$3.5}  \\
      WebcamSurf600-DslrSurf800    &  18.5$\pm$5.9  &  27.7$\pm$4.6  &  28.4$\pm$3.0           &  \textbf{31.3$\pm$3.9}  \\
      WebcamSift900-DslrSurf600    &  13.1$\pm$2.4  &  25.2$\pm$5.6  &  27.5$\pm$4.1           &  \textbf{28.6$\pm$2.9}  \\
      WebcamSift900-DslrSurf800    &  14.5$\pm$3.4  &  25.6$\pm$4.5  &  \textbf{26.9$\pm$4.3}  &  \textbf{26.7$\pm$4.3}  \\
      \hline
    \end{tabular}
  \end{center}
  \caption{Object recognition accuracy. The best performances are highlighted.}
  \label{tab:objrec}
\end{table}

According to table \ref{tab:objrec}, we still observe the ineffectiveness of CCA+NN, and the accuracy improvement with an additional ITML metric learning step. 
Moreover, ARC outperforms CCA+ITML as usual on almost all groups. 
However, superiority of MLHD is not very significant comparing with ARC algorithm in this experiment. 
Although, on those groups whose target domains are DSLR, the performances of MLHD are generally better than those of ARC, the two algorithms' performances are comparable on those groups whose target domains are Webcam.
The reason may lie in that the training samples in target domains in this experiment are not enough due to the limited number of collected images. 
Note that only 11 samples per class are used. 
Consequently, both priors might be aligned biasedly forsuch small number of training samples.

{
\section{Conclusion}
\label{sec:conclusion}
In this paper, we proposed the MLHD model to learn a metric defined across the heterogeneous source and target domains, which is seldon touched to the best of our knowledge.
The proposed model aligns both the priors and posteriors in the source and the target domains at the same time.
Then we show that our model can be reparametrized into a single PSD matrix and use a LogDet function to regularize the model for the convenience of optimization. 
In the following, we give out the optimization method based on Bregman Projection method.
Next, we also show that the model can be easily kernelized by solving an equivalent optimization problem.
Finally, we validate its effectiveness on the multilingual retrieval task and the object recognition task under various situations.
}

\section{Acknowledge}
This work was supported in part by the NSFC of China Grant Nos. 61170151 and 60973097.

\bibliography{HSML}{}
\bibliographystyle{plain}
\end{document}